\documentclass{article}
\usepackage{arxiv}
\usepackage{amsmath}
\usepackage[utf8]{inputenc} 
\usepackage[T1]{fontenc}    
\usepackage{hyperref}       
\usepackage{url}            
\usepackage{booktabs}       
\usepackage{amsfonts}       
\usepackage{nicefrac}       
\usepackage{microtype}      
\usepackage{lipsum}		
\usepackage{graphicx}
\usepackage{natbib}
\usepackage{doi}
\usepackage{algorithm}
\usepackage{algorithmic}
\usepackage{amsthm}

\newtheorem{theorem}{Theorem}[section]  

\title{Robust Variational Model Based Tailored UNet: Leveraging Edge Detector and Mean Curvature for Improved Image Segmentation}

\author
{ 
Kaili Qi\\
    Department of Mathematical Sciences, Tsinghua University \\
    Beijing 100084, China\\
\texttt{qkl21@mails.tsinghua.edu.cn}
\And
    Zhongyi Huang \thanks{Indicates Corresponding Author.}\\
     Department of Mathematical Sciences, Tsinghua University\\
     Beijing 100084, China\\
\texttt{zhongyih@mail.tsinghua.edu.cn}
\And
Wenli Yang \\
School of Mathematics, China University of Mining and Technology \\
Xuzhou 221116, China \\
\texttt{yangwl19@cumt.edu.cn}
}

\begin{document}
\maketitle

\begin{abstract}
	To address the challenge of segmenting noisy images with blurred or fragmented boundaries, this paper presents a robust version of Variational Model Based Tailored UNet (VM\_TUNet), a hybrid framework that integrates variational methods with deep learning. The proposed approach incorporates physical priors, an edge detector and a mean curvature term, into a modified Cahn–Hilliard equation, aiming to combine the interpretability and boundary-smoothing advantages of variational partial differential equations (PDEs) with the strong representational ability of deep neural networks. The architecture consists of two collaborative modules: an $F$ module, which conducts efficient frequency domain preprocessing to alleviate poor local minima, and a $T$ module, which ensures accurate and stable local computations, backed by a stability estimate. Extensive experiments on three benchmark datasets indicate that the proposed method achieves a balanced trade-off between performance and computational efficiency, which yields competitive quantitative results and improved visual quality compared to pure convolutional neural network (CNN) based models, while achieving performance close to that of transformer-based method with reasonable computational expense.
\end{abstract}

\keywords{ Image segmentation\and deep learning \and variational model \and UNet}

\section{Introduction}
Image segmentation is a core and essential task in image processing, which serves as a key component in many real-world applications. It is widely used in diverse fields such as video surveillance, satellite remote sensing, and medical image analysis \cite{Mudassar2025, PanZhuokun2020, Ravikumar2025}. The goal of this task is to assign a category label to every pixel in the input image, thereby isolating target objects and distinguishing them from the background through fine-grained partitioning. Despite significant advances in segmentation techniques, whether through classical algorithms or modern deep neural networks, accurately delineating objects in images affected by noise, ambiguous boundaries, or fragmentation remains a persistent and largely unresolved challenge.

Traditional models for image segmentation have largely relied on model-driven strategies, with variational methods being among the most prominent. These approaches frame segmentation as an energy minimization problem, where the objective integrates terms that encourage smooth boundaries, uniform regions, and proper edge localization. Such formulations produce geometrically consistent results with clear theoretical grounding. Even with the emergence of data-driven deep learning techniques, these classical energy-based methods remain valuable, offering interpretability and robust performance, especially in scenarios with limited training data or when strict control over topology and boundary smoothness is required. Traditional variational segmentation methods are generally divided into four categories according to the design and objectives of their energy functionals: region-based \cite{Chan2001,mumford1989}, boundary-based \cite{Stefaniak2025,WuXiangqiong2025}, shape-prior-based \cite{Rousson2002,Tsai2003}, and texture-based models \cite{Wangguodong2014,wangzhenzhou2016}. Ali, Rada, and Badshah \cite{Ali2018} proposed a novel variational level-set model for automated segmentation of fine image details, which robustly handles intensity inhomogeneity and noise through a denoising-constrained surface and a fidelity term. Niu et al. \cite{NIU2017} formulated a novel region-based active contour model that introduces a local similarity factor based on spatial distance and intensity difference to extract object boundaries under unknown noise. To reliably partition imagery contaminated by significant noise, Li et al. \cite{Li2019} introduced a point distance shape constraint and integrated it into a level set framework for the segmentation of objects with diverse shapes. Recently, Jiang et al. \cite{JIANG2026} designed an active contour model utilizing frequency domain information that employs low-frequency data for robust contour initialization and high-frequency components as energy functional weights, significantly improving segmentation accuracy, efficiency, and stability in noisy medical vessel images. However, the high computational cost of variational models precludes their real-time use for segmenting large or noisy images, despite the availability of efficient solvers. Additionally, variational image segmentation methods suffer from sensitivity to initialization, susceptibility to local minima, parameter tuning challenges, vulnerability to noise and intensity inhomogeneity, limited ability to handle complex scenes, and increased complexity when extended to multi-class or high-dimensional data.

With the accelerated advancement of GPU technology, an increasing array of data-driven convolutional neural networks has emerged in recent years for the task of image segmentation. These models have demonstrated remarkable results, largely attributable to their enhanced capacity for feature representation \cite{Asgari2021,Yangsean2024,Lixiaomeng2018}. For example, Ronneberger, Fischer, and Brox \cite{Ronneberger2015} proposed UNet, a pioneering convolutional network architecture featuring a symmetric encoder-decoder structure with skip connections for precise biomedical image segmentation. Following the success of the original UNet architecture, numerous variants such as Segnet, UNet++ and DeepLabV3+ have emerged, building upon its foundation to achieve significant improvements in image segmentation performance \cite{Badrinarayanan2017,Chen2018,Zhou2018}. Chen et al. \cite{Chenjieneng} designed TransUNet, a pioneering hybrid architecture that leverages global context modeling of vision transformers (ViTs) and the precise localization of UNet for superior medical image segmentation. Subsequently, Swin-UNet \cite{Hu2022} is a pure transformer-based U-shaped architecture that utilizes transformer blocks in both encoder and decoder for medical image segmentation, as it effectively captures hierarchical features. Recently, Segment Anything Model (SAM) \cite{ZHANG2024} is also a transformer-based general-purpose segmentation model, distinct from CNNs. It aims for zero-shot generalization and prompt-driven segmentation, which diverges fundamentally from CNN’s local feature extraction approach. However, SAM suffers from high computational costs due to its large parameter count, and may produce inaccurate masks with jagged edges or lost details for blurry, complex, or sub-pixel structures. Although most CNN-based and transformer-based methods perform competently across a broad range of segmentation tasks, they occasionally produce anomalous results, such as isolated outliers, when segmenting images with significant noise, particularly under conditions of limited training data.

In order to combine the advantages of traditional variational methods and deep learning approaches in the field of image segmentation, many hybrid methods have been proposed recently. Tai, Liu and Chan \cite{Tai2024} proposed PottsMGNet, a multigrid unsupervised neural network that integrates the classical two-phase Potts model as a regularizer to achieve remarkable performance for image segmentation. Similarly, based on the Potts model, Liu et al. \cite{Liu2024} proposed the Double-well Net (DN), which likewise integrates the double-well functional into a deep neural network to achieve more precise and reliable segmentation results. Further, inverse evolution layers (IELs) are novel neural modules that integrate the concept of inverse problems from mathematical physics into deep learning, which enables networks to implicitly learn and reverse evolutionary processes for enhanced output stability and accuracy \cite{Liu2025}. Recently, Jin et al. \cite{Jin2024} proposed a regularized CNN that integrates a geodesic active contour (GAC) model and a learned edge predictor into an end-to-end network for robust and interpretable segmentation of low-contrast and noisy images. Zhang et al. \cite{Zhang2025} designed LMS-Net, a deep unfolding network that integrates a learned Mumford-Shah model with prototype learning to enhance interpretability and structural accuracy in few-shot medical image segmentation. However, in noisy practical scenarios, both traditional variational and hybrid methods struggle to preserve subtle boundaries due to the limitations of energy functionals corresponding to low-order PDEs.

In this paper, inspired by previous studies \cite{Han2014,Liu2024,WEN2022,Zhuwei2012} and to effectively leverage variational method's benefits of strong mathematical interpretability and its capacity to preserve elongated, fragmented, or blurred boundaries under noisy conditions, we propose a novel segmentation network. Our main contributions are as follows.
\begin{itemize}
	\item Our model combines variational methods and deep learning by incorporating physical priors, an edge detector and a mean curvature term. Instead of relying purely on data-driven learning, it integrates the interpretability of variational PDEs with the representation capacity of deep networks. This hybrid design retains both the feature learning strength of deep learning and beneficial variational properties, including transparency, boundary smoothness, and noise robustness.
	\item $F$ and $T$ modules in the model work collaboratively, with the former dedicated to efficient preprocessing and the latter to accurate and stable local computation. Together, they establish a solid foundation for the final segmentation outcome. The $F$ module operates in the frequency domain, which improves computational efficiency, helps avoid suboptimal local minima and provides a better initial state for subsequent optimization. The $T$ module contributes improved local accuracy, stability, and robustness, where a conditional numerical stability theorem is provided along with its proof, offering theoretical support for the reliability of the algorithm.
	\item Extensive experiments on three different datasets with varying Gaussian noise levels demonstrate the feasibility and robustness of the proposed method. With parameter count and per-epoch runtime remaining within a comparable range, our approach achieves a reasonable balance between performance and efficiency. It also yields competitive quantitative results and visual quality, which performs comparably to transformer-based Swin-UNet while showing improvements over pure CNN models.
\end{itemize}
The rest of this paper is organized as follows. Our model and its corresponding algorithm are shown in Section~\ref{main}. The datasets and experimental results are presented in Section ~\ref{experiments}. Finally, we draw a conclusion in Section ~\ref{conclusions}.

\section{Model and algorithm}\label{main}
Building upon the work of Bertalmio et al. \cite{bertalmio2000}, we propose a modified Cahn-Hilliard equation featuring a weight term and a mean curvature term for noisy image segmentation
{\small
	\begin{equation}\label{original}
		\dfrac{\partial u}{\partial t}=-\Delta\big[\varepsilon\nabla\cdot\left(g\nabla u\right)-\dfrac{2}{\varepsilon}W^{\prime}(u)\big]-\lambda\big[u(f-c_1)^2-(1-u)(f-c_2)^2\big]+\mu\nabla\cdot\left(\dfrac{\nabla u}{|\nabla u|}\right),
\end{equation}}
which is equivalent to
\begin{equation}\label{original model}
	\begin{split}
		\dfrac{\partial u}{\partial t}&=-\Delta\bigg(\varepsilon\nabla\cdot\left(g\nabla u\right)\bigg)+\mu\nabla\cdot\bigg(\dfrac{\nabla u}{|\nabla u|}\bigg)\\
		&-\lambda\big[u(f-c_1)^2-(1-u)(f-c_2)^2\big]+\dfrac{2}{\varepsilon}\Delta\big(W^{\prime}(u)\big).
	\end{split}
\end{equation}
This equation is subject to the following boundary conditions
\begin{equation}\label{boundary conditions}
	\dfrac{\partial u}{\partial \boldsymbol{n}}=\dfrac{\partial\big(\nabla\cdot(g(\nabla f)\nabla u)\big)}{\partial \boldsymbol{n}}=0,
\end{equation}
where $f:\Omega\subset\mathbb{R}^2\rightarrow \mathbb{R}$ is the input image and $\Omega$ denotes the image domain. The parameters $\varepsilon, \lambda, \mu>0$ control the behavior of the model. $W(t)=(t^2-1)^2$ is a Lyapunov functional, and $u$ is an evolving phase-field function whose steady-state solution defines the segmentation. The operators $\nabla$ and $\Delta$ represent the gradient and Laplacian, respectively. The function $g(\nabla f) = 1/(1+\beta|\nabla f|^2)$, parameterized by $\beta > 0$, serves as an edge detector that diminishes near image boundaries. For the implementation of $g(\nabla f)$, we consider the finite difference method, specifically using central differencing at non-boundary points and first-order differencing at boundary points. $\nabla\cdot\left(\nabla u/|\nabla u|\right)$ is the mean curvature term which is specifically designed to handle noise. $c_1$ and $c_2$ are the average intensities inside and outside the target region. Inspired by \cite{Caselles1997,wang2022}, we integrate the edge detector and the mean curvature term to handle the problem of noisy image segmentation.

The remainder of this section is structured as follows. Subsection~\ref{FSM} presents the Fourier spectral method (FSM) and implementation details. Subsection~\ref{TFPM} describes the tailored finite point method (TFPM), including implementation and an analysis of stability. Finally, a detailed introduction to the robust version of VM\_TUNet is provided in Subsection~\ref{robust}.

\subsection{Fourier spectral method}\label{FSM}
For $\Omega=[0,L_x]\times[0,L_y]$, we discretize it into $N_x\times N_y$ grid points and the grid spacing and coordinates are $\Delta x=L_x/(N_x-1),\text{ }\Delta y=L_y/(N_y-1)$ and $x_{i}=i\Delta x,\text{ }y_{j}=j\Delta y,\text{ }i=0,\ldots,N_x-1,\text{ }j=0,\ldots,N_y-1$. Employing the Fourier method inherently assumes the application of periodic boundary conditions during numerical implementation, which is also a approach frequently adopted in image processing contexts
\begin{equation}\notag
	u(0,y)=u(L_x,y),\textbf{ }0\leq y\leq L_y;\textbf{ }u(x,0)=u(x,L_y),\textbf{ }0\leq x\leq L_x.
\end{equation}
The corresponding frequency-domain wavenumbers in both directions are precomputed during initialization to facilitate subsequent computations
\begin{equation}\notag
	k_x[m]=\dfrac{2\pi}{L_x}\cdot
	\begin{cases}
		m,&m=0,1,\ldots,\lfloor\dfrac{N_x}{2}\rfloor,\\
		m-N_x,&m=\lfloor\dfrac{N_x}{2}\rfloor+1,\ldots,N_x-1,
	\end{cases}
\end{equation}
and
\begin{equation}\notag
	k_y[n]=\dfrac{2\pi}{L_y}\cdot
	\begin{cases}
		n,&n=0,1,\ldots,\lfloor\dfrac{N_y}{2}\rfloor,\\
		n-N_y,&n=\lfloor\dfrac{N_y}{2}\rfloor+1,\ldots,N_y-1,
	\end{cases}
\end{equation}
and the squared magnitude of the wave number is given by $|\mathbf{k}[m,n]|^2=(k_x[m])^2+(k_y[n])^2$.

Given a discrete function $f(x,y)$, the forward discrete Fourier transform (FDFT) is defined as
\begin{equation}\notag
	\widehat{f}[k_x,k_y]=\mathcal{F}\{f\}[k_x,k_y]=\sum\limits_{i=0}^{N_x-1}\sum\limits_{j=0}^{N_y-1}f(x_i,y_j)\exp\{-2\pi \mathbf{i}\bigg(\dfrac{k_x i}{N_x}+\dfrac{k_y j}{N_y}\bigg)\},
\end{equation} and the inverse discrete Fourier transform (IDFT) is defined as
\begin{equation}\notag
	f(x_i,y_j)=\mathcal{F}^{-1}\{\widehat{f}\}(x_i,y_j)=\dfrac{1}{N_xN_y}\sum\limits_{k_x=0}^{N_x-1}\sum\limits_{k_y=0}^{N_y-1}\widehat{f}[k_x,k_y]\exp\{2\pi \mathbf{i}\bigg(\dfrac{k_x i}{N_x}+\dfrac{k_y j}{N_y}\bigg)\},
\end{equation}where $\mathbf{i}$ denotes the imaginary unit. For the gradient operator, we consider 
\begin{equation}\notag
	\nabla u=\bigg(\dfrac{\partial u}{\partial x},\dfrac{\partial u}{\partial y}\bigg)=\big(\mathcal{F}^{-1}\{\mathbf{i}k_x\widehat{u}\},\mathcal{F}^{-1}\{\mathbf{i}k_y\widehat{u}\}\big).
\end{equation}For the divergence operator, we consider the vector field $\mathbf{v}=(v_x,v_y)$ and then we have
\begin{equation}\notag
	\nabla\cdot\mathbf{v}=\dfrac{\partial v_x}{\partial x}+\dfrac{\partial v_y}{\partial y}=\mathcal{F}^{-1}\{\mathbf{i}(k_x\widehat{v_x}+k_y\widehat{v_y})\}.
\end{equation}
For the Laplacian operator, we consider 
\begin{equation}\notag
	\Delta u=\dfrac{\partial^2 u}{\partial x^2}+\dfrac{\partial^2 u}{\partial y^2}=\mathcal{F}^{-1}\{-|\mathbf{k}|^2\widehat{u}\}.
\end{equation}

For the first term of \eqref{original model}, i.e., the diffusion term with each time step $\Delta t=t^{n+1}-t^{n}$ and $u(t^n)=u^n,\text{ }u^0=f$. Let $A=\nabla\cdot(g\nabla u^n)$, then 
\begin{equation}\notag
	\widehat{A}= \mathbf{i} \left[ k_x \cdot \mathcal{F} \{ g(\nabla f) \cdot \mathcal{F}^{-1}\{\mathbf{i} k_x \widehat{u}^n\} \} + k_y \cdot \mathcal{F} \{ g(\nabla f) \cdot \mathcal{F}^{-1}\{\mathbf{i} k_y \widehat{u}^n\} \} \right],
\end{equation} 
and $\mathcal{F}\{\Delta A\}=-|\mathbf{k}|^2\widehat{A}$. For the second term of \eqref{original model}, i.e., the mean curvature term, let $B=\nabla\cdot\left(\nabla u^n/(|\nabla u^n|+\delta)\right)$ with $\delta=10^{-8}$, then
\begin{equation}\notag
	\begin{split}
		\widehat{B} &= \mathbf{i} k_x \cdot \mathcal{F} \{ \frac{ \mathcal{F}^{-1}\{\mathbf{i} k_x \widehat{u}^n\} }{ \sqrt{ (\mathcal{F}^{-1}\{\mathbf{i} k_x \widehat{u}^n\})^2 + (\mathcal{F}^{-1}\{\mathbf{i} k_y \widehat{u}^n\})^2 }+ \delta }  \} \\
		&+\mathbf{i}k_y \cdot \mathcal{F} \{ \frac{ \mathcal{F}^{-1}\{\mathrm{i} k_y \widehat{u}^n\} }{ \sqrt{ (\mathcal{F}^{-1}\{\mathbf{i} k_x \widehat{u}^n\})^2 + (\mathcal{F}^{-1}\{\mathbf{i} k_y \widehat{u}^n\})^2 }+ \delta }  \},
	\end{split}
\end{equation} 
The transition state update in the frequency domain is performed as follows
\begin{equation}\notag
	\widehat{u}^{(1)}=\widehat{u}^n+\Delta t\big[\varepsilon|\mathbf{k}|^2\widehat{A} + \mu\widehat{B}\big].
\end{equation}
For the treatment of data fidelity item of nonlinear terms, we denote it by \[D=-\lambda\left[u^n(f-c_1^n)^2-(1-u^n)(f-c_2^{n})^2\right],\] and we update two parameters $c_1$ and $c_2$ by
\begin{equation}\notag
	c_1^n=\dfrac{\sum u^n\cdot f}{\sum u^n+\delta},\textbf{ }c_2^n=\dfrac{\sum(1-u^n)\cdot f}{\sum(1-u^n)+\delta}.
\end{equation}For the Lyapunov functional term $W(u^n)$, and we have \[P=\dfrac{2}{\varepsilon}\Delta\left(W^{\prime}(u^n)\right)=\dfrac{2}{\varepsilon}\mathcal{F}^{-1}\{-|\mathbf{k}|^2\mathcal{F}\{W^{\prime}(u^n)\}\}.\] Overall update in the frequency domain is performed as follows
\begin{equation}\notag
	\begin{aligned}
		\widehat{u}^{n+1}&=\widehat{u}^{(1)}+\Delta t\mathcal{F}\{D+P\},\\
		u^{n+1}&=\mathcal{F}^{-1}\{\widehat{u}^{n+1}\},\\
		u^{n+1}&=\max\left(0,\min(1,u^{n+1})\right).
	\end{aligned}
\end{equation}
The algorithm terminates, whose output is denoted by $u^{final}$, when the relative $L^2$-norm change of the iterative solution, i.e., $\|u^{n+1}-u^{n}\|_{2}/\|u^{n}\|_{2}$ falls below a preset tolerance $tol$ or when the number of iterations $n$ exceeds the specified maximum iteration count $n_{\max}$. For details of the Fourier spectral method, see Algorithm~\ref{FSMalgorithm}.
\begin{algorithm}[htbp]
	\caption{Fourier Spectral Method for the Modified Cahn-Hilliard Equation}
	\label{FSMalgorithm}
	\begin{algorithmic}
		\REQUIRE $f(x,y),\text{ }g(\nabla f),\text{ }\varepsilon, \lambda, \mu, \beta>0$,
		\STATE \textbf{Initialization:}
		\STATE $x_i \gets i\Delta x$, $y_j \gets j\Delta y,\text{ }n \gets 0,\textbf{ }\widehat{u}^0 \gets \mathcal{F}\{f\}$ and $k_x[m], k_y[n], |\mathbf{k}[m,n]|^2$
		
		\WHILE{$n < n_{\text{max}}$ and $\|u^{n+1} - u^n\|_2/\|u^n\|_2 > \text{tol}$}
		
		\STATE Diffusion Term: $A \gets \nabla \cdot (g(\nabla f) \nabla u^n),\text{ }\widehat{A} \gets \mathcal{F}\{A\}$
		\STATE Curvature Term: $B \gets \nabla \cdot \left( \dfrac{\nabla u^n}{|\nabla u^n| + \delta} \right),\text{ }\widehat{B} \gets \mathcal{F}\{B\}$
		\STATE Transition State: $\widehat{u}^{(1)} \gets \widehat{u}^n + \Delta t \left( \varepsilon |\mathbf{k}|^2 \widehat{A} + \mu\widehat{B} \right)$
		\STATE Parameter Update: $c_1^n \gets \dfrac{\sum u^n\cdot f}{\sum u^n + \delta},\text{ }c_2^n \gets \dfrac{\sum (1-u^n)\cdot f}{\sum (1-u^n) + \delta}$
		\STATE Data Fidelity: $D \gets -\lambda \left[ u^n(f-c_1^n)^2 - (1-u^n)(f-c_2^n)^2 \right]$
		\STATE Lyapunov Functional:  $P \gets \dfrac{2}{\varepsilon} \mathcal{F}^{-1}\{-|\mathbf{k}|^2 \mathcal{F}\{W'(u^n)\}\}$
		
		\STATE $\widehat{u}^{n+1} \gets \widehat{u}^{(1)} + \Delta t \mathcal{F}\{D+P\},\text{ }u^{n+1} \gets \mathcal{F}^{-1}\{\widehat{u}^{n+1}\},\text{ }u^{n+1} \gets \max(0, \min(1, u^{n+1}))$
		\STATE $n \gets n + 1$
		\ENDWHILE
		\STATE \textbf{Return} Final Segmentation $u^{\text{final}} \gets u^{n+1}$
	\end{algorithmic}
\end{algorithm}

\subsection{Tailored finite point method}\label{TFPM}
We proceed to consider a reduced form of original Cahn-Hilliard equation, obtained by eliminating select complex components from its full formulation, under consideration that FSM serves as a preprocessing mechanism for edge detection and noise reduction, i.e.,
\begin{equation}\label{wang}
	\dfrac{\partial u}{\partial t}=-\Delta\left(\varepsilon\Delta u-\dfrac{2}{\varepsilon}W^{\prime}(u)\right)-\lambda\big[u(f-c_1)^2-(1-u)(f-c_2)^2\big].
\end{equation}
Let $ v=\varepsilon\Delta u-\dfrac{2}{\epsilon}W^{\prime}(u)$, then we have
\begin{equation}\label{ouhe}
	\begin{cases}
		v=\varepsilon\Delta u-\dfrac{2}{\epsilon}W^{\prime}(u),\\
		u_t=-\Delta v-\lambda\big[u(f-c_1)^2-(1-u)(f-c_2)^2\big],
	\end{cases}
\end{equation}
where $\partial u/\partial \boldsymbol{n}=\partial v/\partial \boldsymbol{n}=0$. Since small parameter $\varepsilon$ appears only in the first equation, we focus on formulating the TFPM scheme for it, which first requires linearizing the nonlinear term $W^{\prime}(u)$
\begin{equation}\notag
	W^{\prime}(u)=4u^3-4u=(4u^2+2)u-6u,
\end{equation}
where $4u^2+2$ and $-6u$ can be considered as constants locally. Furthermore, we can use the following scheme to discretize \eqref{ouhe}, which is as follows
\begin{equation}\label{fenxi}
	\begin{cases}
		\dfrac{u^{n+1}-u^n}{\tau}=-\Delta v^{n+1}-\lambda\left[u^{n+1}\left(f-c_1\right)^2-(1-u^{n+1})\left(f-c_2\right)^2\right],\\
		v^{n+1}=\varepsilon\Delta u^{n+1}-\dfrac{2}{\varepsilon}\left[(4(u^n)^2+2)u^{n+1}-6u^{n}\right],
	\end{cases}
\end{equation}
where each time step is $\tau=t^{n+1}-t^n$ with $u(t^n)=u^n,\textbf{ }u^{0}=\mathrm{Sig}(W^0*u^{final}+b^0)$. $\mathrm{Sig}$ represents a sigmoid function, $W^0$ is set as a convolutional kernel and $b^0$ is a bias term. After linearization, the coefficients of part $W^{\prime}(u)$ are in explicit format, and the rest are in implicit format. In order to capture some boundary and/or interior layers, we select some special basis functions to interpolate $u$. Let $\Delta x=\Delta y=h$ and $u_{ij}^{n}=u^n(ih,jh)$, then the approximate solution $u_h$ to the reduced equation satisfies
\begin{equation}\notag
	u_h(x,y)\in\dfrac{\dfrac{12u_{ij}^{n}}{\varepsilon}-v_{ij}^{n+1}}{\dfrac{4(2(u_{ij}^{n})^2+1)}{\varepsilon}}+\mathrm{span}\{e^{\xi x},e^{-\xi x},e^{\xi y},e^{-\xi y}\},
\end{equation}
where 
\begin{equation}\notag
	\xi=\dfrac{2\sqrt{2(u_{ij}^{n})^2+1}}{\varepsilon}.
\end{equation}
Due to the characteristics of image processing problems, we can actually set the spatial step size to $h=1$. Following \cite{yang2019}, one can easily get
\begin{equation}\notag
	u_{x,i+\frac{1}{2},j}=\xi\dfrac{u_{i+1,j}-u_{ij}} {e^{\frac{\xi h}{2}}-e^{-\frac{\xi h}{2}}},\textbf{ }u_{y,i,j+\frac{1}{2}}=\xi\dfrac{u_{i,j+1}-u_{ij}} {e^{\frac{\xi h}{2}}-e^{-\frac{\xi h}{2}}}.
\end{equation}
Then TFPM scheme of the second equation in \eqref{fenxi} is
\begin{equation}\label{first}
	\begin{split}
		&\dfrac{4}{\varepsilon}\left(2\left(u_{ij}^{n}\right)^2+1\right)u_{ij}^{n+1}-\dfrac{\varepsilon\xi}{h\left(e^{\frac{\xi h}{2}}-e^{-\frac{\xi h}{2}}\right)}\\
		&\cdot\left(u_{i+1,j}^{n+1}+u_{i-1,j}^{n+1}+u_{i,j+1}^{n+1}+u_{i,j-1}^{n+1}-4u_{i,j}^{n+1}\right)
		+v_{ij}^{n+1}=\dfrac{12}{\varepsilon}u_{ij}^{n},
	\end{split}
\end{equation}
Let $H_1(f)=(f-c_1)^2+(f-c_2)^2$ and $H_2(f)=(f-c_2)^2$, then $\lambda\big[u(f-c_1)^2-(1-u)(f-c_2)^2\big]=\lambda\left(u\cdot H_1-H_2\right)$. Thus, we have
\begin{equation}\label{second}
	\dfrac{u^{n+1}-u^n}{\tau}=-\Delta v^{n+1}-\lambda\left(u^{n+1}\cdot H_1-H_2\right),
\end{equation}
where $\Delta v^{n+1}$ can also be computed by $\Delta v_{ij}^{n+1}=v_{i+1,j}^{n+1}+v_{i-1,j}^{n+1}+v_{i,j+1}^{n+1}+v_{i,j-1}^{n+1}-4v_{ij}^{n+1}$. For the above two equations \eqref{first} and \eqref{second}, we can use biconjugate gradient stabilized (BICGSTAB) method \cite{BiCGSTAB} to solve them by stacking unknown variables along the channel dimension. After evolving the system within the time domain $[0, T]$, we take the state $u_T=u(x, T)$ at the final time $T$ to be our segmentation function. Denote $\left\langle\cdot,\cdot\right\rangle$ as the inner product corresponding to the $L^2$ norm. The following is a stability estimate for the TFPM computational scheme.

\begin{theorem}[Stability Estimate]\label{thm}
	Assuming that $W^{\prime\prime}(u^{n})\leq K$ holds, where $K$ is some positive constant, then when the time step $\tau$ and parameters $\varepsilon,\ \lambda$ satisfy some conditions, the boundedness of $u^{n+1}$ and $\Delta u^{n+1}$ is controlled by the boundedness of $u^n$ and $\Delta u^{n}$ on the interval $[0,T]$ with bounded $u^{0}$ and $\Delta u^{0}$.
	
\end{theorem}
\begin{proof}\label{proof}
	Multiplying the equation \eqref{fenxi} by $u^{n+1}$ and integrating over the area $\Omega$, we have 
	\begin{equation}
		\begin{split}
			&\dfrac{1}{\tau}\left(\|u^{n+1}\|^2-\left\langle u^n,\ u^{n+1}\right\rangle\right)+\varepsilon\|\Delta u^{n+1}\|^2
			=\dfrac{8}{\varepsilon}\left\langle u^{n+1}\left(\left(u^n\right)^2-1\right),\ \Delta u^{n+1}\right\rangle\\
			&+\dfrac{12}{\varepsilon}\left\langle u^{n+1},\ \Delta u^{n+1}\right\rangle
			+\dfrac{12}{\varepsilon}\left\langle-\Delta u^{n},\ u^{n+1}\right\rangle+\lambda\left\langle H_2-u^{n+1}\cdot H_1,u^{n+1}\right\rangle.
		\end{split}
	\end{equation}
	By the Young's inequality and integration by part, and according to the definitions of $H_1,\ H_2$, there exists a positive constant $C_1$ such that
	\begin{equation}
		\begin{split}
			&\dfrac{1}{2\tau}\left(\|u^{n+1}\|^2-\|u^n\|^2\right)+\varepsilon\|\Delta u^{n+1}\|^2\leq\dfrac{8}{\varepsilon}\left\langle u^{n+1}\left(\left(u^n\right)^2-1\right),\ \Delta u^{n+1}\right\rangle\\
			&-\dfrac{12}{\varepsilon}\|\nabla u^{n+1}\|^{2}+\dfrac{6}{\varepsilon}(\|\Delta u^n\|^2+\| u^{n+1}\|^2)+\dfrac{C_1\lambda}{2}\left(\|f\|^2-\|u^{n+1}\|^2\right).
		\end{split}
	\end{equation}
	Under assumption $W^{\prime\prime}(u^{n})\leq K$, i.e., $(u^n)^2\leq(K+4)/12$, and by the Cauchy-Schwarz inequality and the Young's inequality, we have
	\begin{equation}
		\begin{split}
			\dfrac{8}{\varepsilon}\left\langle u^{n+1}\left(\left(u^n\right)^2-1\right),\ \Delta u^{n+1}\right\rangle&\leq \dfrac{8}{\varepsilon}\|u^{n+1}\|\cdot\|\Delta u^{n+1}\|\cdot \dfrac{K-8}{12}\\
			&\leq\dfrac{K-8}{3\varepsilon}\left(\gamma\|u^{n+1}\|^2+\dfrac{1}{\gamma}\|\Delta u^{n+1}\|^2\right),
		\end{split}
	\end{equation}
	where $\gamma>0$ is some positive constant and $K>8$. Then
	\begin{equation}
		\begin{split}
			&\dfrac{1}{2\tau}\left(\|u^{n+1}\|^2-\|u^n\|^2\right)+\varepsilon\|\Delta u^{n+1}\|^2\leq \dfrac{K-8}{3\varepsilon}\left(\gamma\|u^{n+1}\|^2+\dfrac{1}{\gamma}\|\Delta u^{n+1}\|^2\right)\\
			&+\dfrac{6}{\varepsilon}(\|\Delta u^n\|^2+\| u^{n+1}\|^2)+C(\Omega,\ f,\ C_1)-\dfrac{C_1\lambda}{2}\|u^{n+1}\|^2.
		\end{split}
	\end{equation}
	After sorting, we have
	\begin{equation}
		\begin{split}
			&	\left(\dfrac{1}{2\tau}-\dfrac{K-8}{3\varepsilon}\gamma-\dfrac{6}{\varepsilon}+\dfrac{C_1\lambda}{2}\right)\|u^{n+1}\|^2+\left(\varepsilon-\dfrac{K-8}{3\varepsilon\gamma}\right)\|\Delta u^{n+1}\|^2\\
			&\leq\dfrac{1}{2\tau}\|u^{n}\|^2+\dfrac{6}{\varepsilon}\|\Delta u^n\|^2+C(\Omega,\ f,\ C_1).
		\end{split}
	\end{equation}
	where when $\dfrac{1}{2\tau}-\dfrac{K-8}{3\varepsilon}\gamma-\dfrac{6}{\varepsilon}+\dfrac{C_1\lambda}{2}>0$ and $\varepsilon-\dfrac{K-8}{3\varepsilon\gamma}>0$, the proof is complete.
\end{proof}

\subsection{Robust VM\_TUNet}\label{robust}
To further improve the performance of the VM\_TUNet network, we incorporate $F$ module and $T$ module. $F$ module, constructed from a series of $B_F$ blocks and grounded in the FSD provides computational efficiency, features an algorithmically simple implementation and helps avoid poor local minima in subsequent optimization or learning processes. Simultaneously, $T$ module, implemented as the TFPM module and containing a sequence of $B_T$ blocks, delivers better local accuracy, improved stability, and robustness. $T$ module processes input image $f$ through an operator $H(f)$, which is formed by combining two UNet-like subnetworks specified by a channel vector $\boldsymbol{c}=[c_1,\ldots,c_S]$ for some positive integers $\{c_s\}_{s=1}^{S}$, where $S$ denotes the number of resolution levels as in \cite{Liu2024}. That is, we represent $H_1(f)-H_2(f)=(f-c_1)^2$ and $H_2(f)=(f-c_2)^2$ by using two neural networks to automatically optimize two positive parameters $c_1$ and $c_2$, thereby avoiding manual tuning and facilitating the application of the BiCGSTAB solver. Overall, we have robust VM\_TUNet as the following formula
\begin{equation}\notag
	P(f)=\textrm{Sig}\left(W*T\left(\textrm{Sig}\left(F(f)\right),H(f)\right)+b\right),
\end{equation}
where we show the architecture of robust VM\_TUNet in Figure~\ref{VMEDMCTNet}.
\begin{figure}[ht]
	\centering\includegraphics[width=1\linewidth]{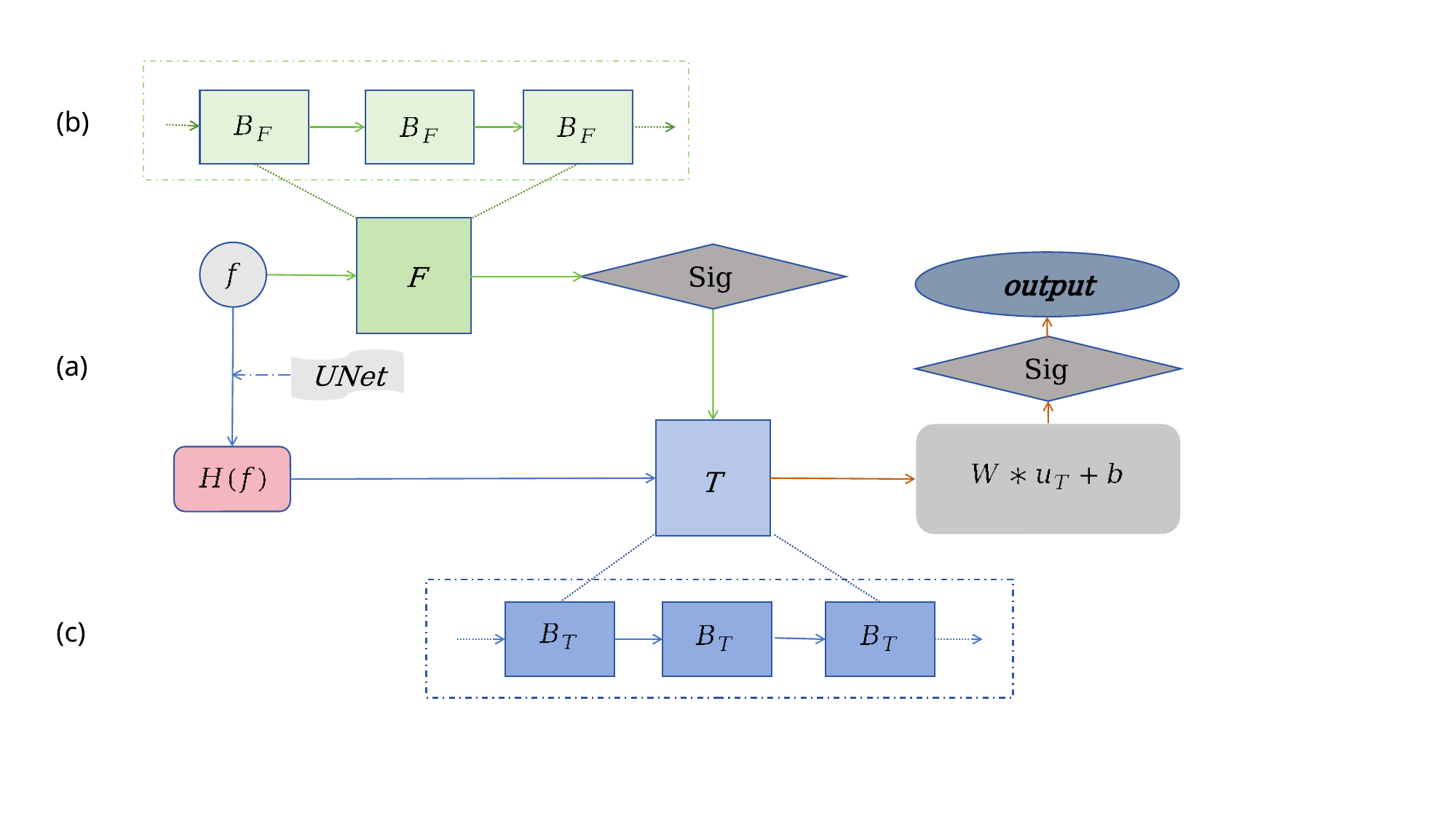}  
	\caption{ robust VM\_TUNet architecture. (a) Main process of robust VM\_TUNet with a convolutional layer followed by a sigmoid function at the end. (b) FSM blocks. (c) TFPM blocks.}  
	\label{VMEDMCTNet}  
\end{figure}
Suppose we are given a training set of images $\{f_{i}\}_{i=1}^{I}$ with their foreground-background segmentation masks $\{g_{i}\}_{i=1}^{I}$. Our goal is to learn a data-driven operator $H(f)$ such that for any image $f$ with properties similar to those in the training set, the steady state of \eqref{original} closely approximates its segmentation mask $g$. Let $\Theta$ denote the set of all parameters of $H(f)$ to be learned from the data. These parameters are determined by solving
\begin{equation}\notag
	\min\limits_{\Theta}\frac{1}{I}\sum\limits_{i=1}^{I}\ell(u(x,T;\Theta,f_{i}),g_{i}),
\end{equation}
where $\ell(\cdot,\cdot)$ is a loss function, such as Hinge loss, Logistic loss, or Binary Cross Entropy (BCE) loss, which measures the differences between its arguments \cite{Loss2004}.

\section{Experimental results}\label{experiments}
This section is dedicated to assessing the segmentation performance of our proposed method and current state-of-the-art segmentation models with a roughly equivalent number of parameters. UNet \cite{Ronneberger2015} is a classic encoder-decoder architecture for medical image segmentation. Its defining feature is the use of skip connections that fuse high-resolution feature maps from the encoder with the upsampled feature maps in the decoder, a design that effectively preserves precise spatial information. UNet++ \cite{Zhou2018} redesigns the skip pathways of the standard UNet by introducing densely connected convolutional blocks and nested skip connections. This architecture aims to bridge the semantic gap between the encoder and decoder features, which leads to more accurate and refined segmentation outputs. SwinUNet \cite{Hu2022} is a purely transformer-based segmentation model built upon the swin transformer. It utilizes a shifted windows mechanism for self-attention to efficiently model long-range dependencies and capture global contextual information, which demonstrates powerful capabilities of transformers in visual tasks. Based on the Potts model and the operator splitting method, Double-well Net (DN) \cite{Liu2024} is a mathematically grounded deep learning framework that incorporates a double-well potential function to drive image segmentation. It leverages a UNet-style architecture to learn region force terms in a data-driven manner, which achieves high accuracy with fewer parameters compared to conventional segmentation networks. The objective of all these methods is to enhance the alignment of segmentation predictions with their corresponding ground truth labels. Our approach leverages advantages of both variational models and deep learning. A detailed comparison of the parameters among these deep learning methods and our model is provided in Table~\ref{VMEDMC parameter}.
\begin{table}[htbp]
	\caption{Comparison of the number of parameters of DN, ours with UNet, UNet++, Swin-UNet, and TransUNet, SAM in this study (M: Represents one million parameters).}
	\label{VMEDMC parameter}
	\centering
	\begin{tabular}{ccccccc}
		\toprule
		\multicolumn{7}{c}{Number of parameters} \\
		\midrule
		DN & Ours & UNet & UNet++ & Swin-UNet & TransUNet & SAM \\
		7.7M & 19.8M & 30.2M & 35.6M & 33.9M & 105.1M & \textgreater{} 600.0M \\
		\bottomrule
	\end{tabular}
\end{table}
The parameter count of our model is comparable to that of these methods. Therefore, for the assessment of segmentation accuracy, UNet, UNet++, Swin-UNet, and DN form the basis of our comparative evaluation. 

To gauge the effectiveness of various segmentation techniques in measurable terms, three widely-adopted performance indicators, each relying on comparison to a ground truth reference, are employed. To holistically assess segmentation quality, two complementary classes of metrics are employed, region-based and boundary-based measures. Region accuracy is quantified using Dice and Jaccard indices, where superior performance is reflected in higher scores \cite{Bertels2019}. Conversely, boundary fidelity is evaluated using 95th percentile Hausdorff Distance (HD95), a metric for which lower values indicate finer alignment with the reference contour \cite{Taha2015}.

\subsection{Implementation}
This segment outlines the specific configurations and dataset characteristics employed in the implementation of the proposed model.

First, during the training process, we use an Adam optimizer with learning rate $10^{-4}$, BCE loss and batch size $8$ for all experiments of this paper. Moreover, in $F$ block whose block numbers are set $5$, we assign $\varepsilon=\lambda=\beta=\mu=1$, $\Delta t=0.1$, $tol=10^{-5}$ and $n_{max}=10$ for each $B_F$ block. Adopting the same channel configuration as DN, we define $\boldsymbol{c}$ in $H(f)$ as $[128,128,128,128,256]$. For $T$ block whose block number are set $5$, we assign $\varepsilon=\lambda=1$, $tol=10^{-5}$, $\tau=0.02$ and the maximum number of iterations are $5, 10, 15, 20,25$ for BiCGSTAB of each $B_T$ block, respectively. We neither need nor wish to devote excessive effort to parameter tuning, as FSD preprocessing already effectively reduces noise and enhances boundary clarity, while TFPM scheme achieves desirable results without requiring full convergence. To ensure a fair and authoritative comparison, all models were evaluated under their default or near-optimal configurations to maximize their performance, with a constraint that the total parameter count for each method should be maintained at a comparable order of magnitude. All experimental computations were executed on an Ubuntu 24.04 system within a Python 3.12 environment utilizing PyTorch framework version 2.8.0, with all neural network operations accelerated on an NVIDIA 5090 GPU.

In this study, all employed datasets were synthetically corrupted by adding zero-mean Gaussian noise with varying standard deviations $\sigma$:
\begin{itemize}
	\item ECSSD comprises 1000 semantically annotated images characterized by complex backgrounds and pixel-wise manual annotations. All samples were resized to 256×256 resolution and trained for 600 epochs, with a partition of 600 images allocated for training and the remaining 400 reserved for testing \cite{Shi2016}.
	\item HKU-IS is designed for visual saliency prediction and comprises 4447 challenging images, most exhibiting either low contrast or containing multiple salient objects. For our experiments, all images were resized to 256×256 resolution and trained for 800 epochs, with 2840 samples used for training and 1607 for testing \cite{Li2015}.
	\item DUT-OMRON dataset consists of 5168 high-quality natural images, each featuring one or more salient objects against diverse and cluttered backgrounds. All images were resized to a resolution of 256×256 and trained for 800 epochs, with 3328 images allocated for training and 1840 for testing \cite{Yang2013}.
\end{itemize}

\subsection{Test on datasets}
This section presents a comparative evaluation of multiple segmentation approaches applied across three distinct benchmark datasets. Figure~\ref{ECSSD SD0.3}~to~Figure~\ref{HKUIS SD0.7} primarily demonstrate segmentation performance on single objects, whereas Figure~\ref{HKUIS SD0.5 1}~to~Figure~\ref{DUT SD0.7} focus on scenarios containing multiple objects.

In Figure\ref{ECSSD SD0.3}, UNet and UNet++ achieve roughly correct object localization but yield somewhat blurry and fragmented boundaries, resulting in partial loss of foreground areas; DN maintains relatively better target completeness, though boundary precision remains limited; both Swin-UNet and our method show improved overall performance compared to the above approaches, while Swin-UNet appears to slightly over-smooth certain edges, which may contribute to minor adhesion and less precise delineation in regions such as the bird's beak. 

\begin{figure}[htbp]
	\centering
	\includegraphics[width=1\linewidth]{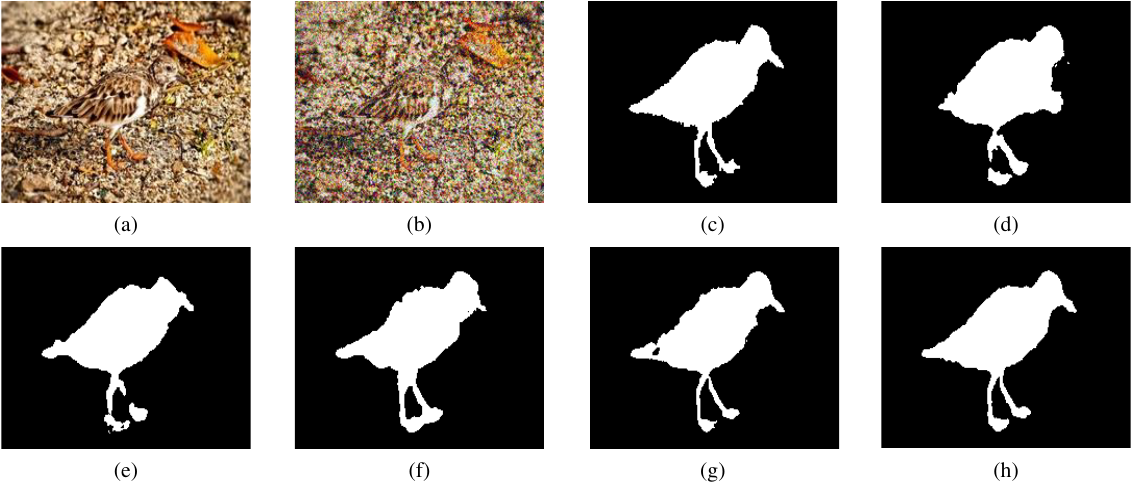}
	\caption{Results of single object of ECSSD (Gaussian noise $\sigma = 0.3$). (a) Original image; (b) Noisy image; (c) Ground Truth; (d) UNet; (e) UNet++; (f) Swin-UNet; (g) DN; (h) Ours.}
	\label{ECSSD SD0.3}
\end{figure}

As shown in Figure~\ref{ECSSD SD0.7}, while UNet, UNet++, and DN maintain approximate object localization, they exhibit noticeable boundary blur, fragmentation, and part adhesion under high-intensity noise interference. In comparison, both Swin-UNet and our method achieve competitive performance. Swin-UNet shows slight local adhesion and some inaccuracy in representing the curved stem of the leftmost glass, while our method appears to mitigate these issues.

\begin{figure}[htbp]
	\centering
	\includegraphics[width=1\linewidth]{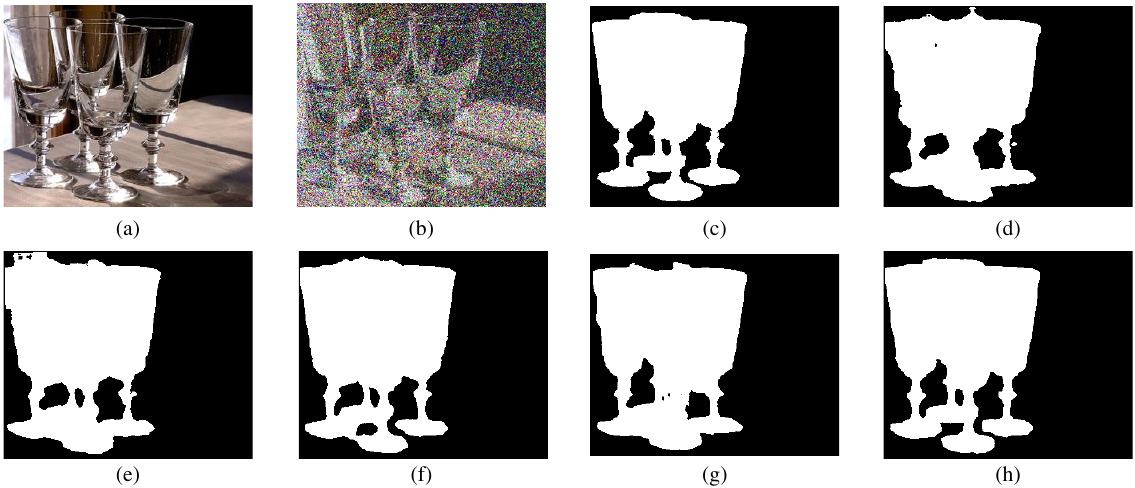}
	\caption{Results of single object of ECSSD (Gaussian noise $\sigma = 0.7$). (a) Original image; (b) Noisy image; (c) Ground Truth; (d) UNet; (e) UNet++; (f) Swin-UNet; (g) DN; (h) Ours.}
	\label{ECSSD SD0.7}
\end{figure}

On both images Figure~\ref{HKUIS SD0.3} and Figure~\ref{HKUIS SD0.7}, all five methods successfully reconstructed the general shapes of the objects. It can be observed that DN and Swin-UNet capture more fine-grained details compared to UNet and UNet++, while ours appears to achieve slightly better performance than DN and Swin-UNet. Specifically, in Figure~\ref{HKUIS SD0.3}, UNet and UNet++ did not fully represent the uneven texture of the pineapple’s surface. In Figure~\ref{HKUIS SD0.7}, these two methods seem to have been affected by branches similar in color to the snail above it and strong noise, resulting in some extraneous regions in the segmentation. DN and Swin-UNet managed to recover more details in Figure~\ref{HKUIS SD0.3}, though still somewhat less completely than our method. Additionally, in Figure~\ref{HKUIS SD0.7}, the snail’s shell and antennae recovered by these two methods are somewhat less distinct.

\begin{figure}[htbp]
	\centering
	\includegraphics[width=1\linewidth]{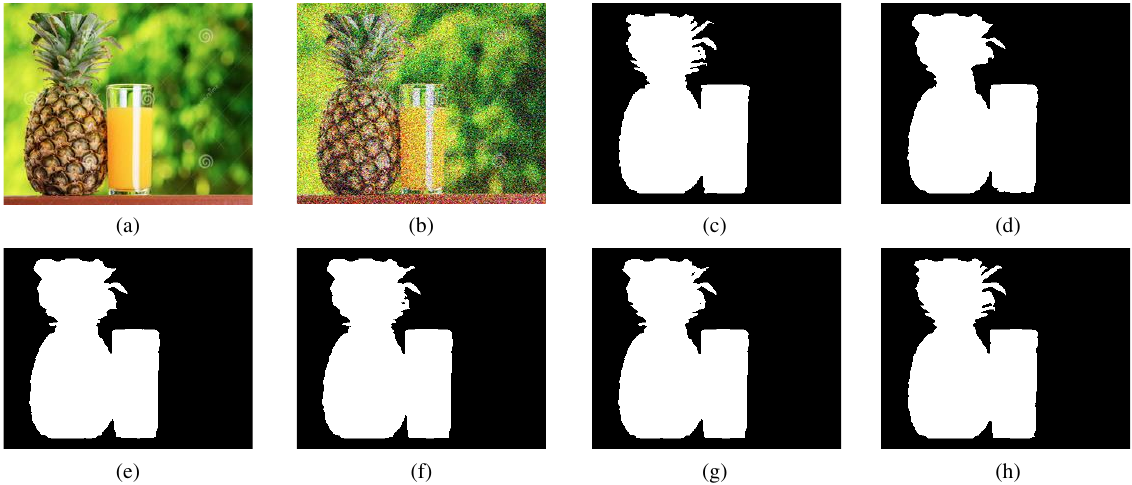}
	\caption{Results of single object of HKU-IS (Gaussian noise $\sigma = 0.3$). (a) Original image; (b) Noisy image; (c) Ground Truth; (d) UNet; (e) UNet++; (f) Swin-UNet; (g) DN; (h) Ours.}
	\label{HKUIS SD0.3}
\end{figure}

\begin{figure}[htbp]
	\centering
	\includegraphics[width=1\linewidth]{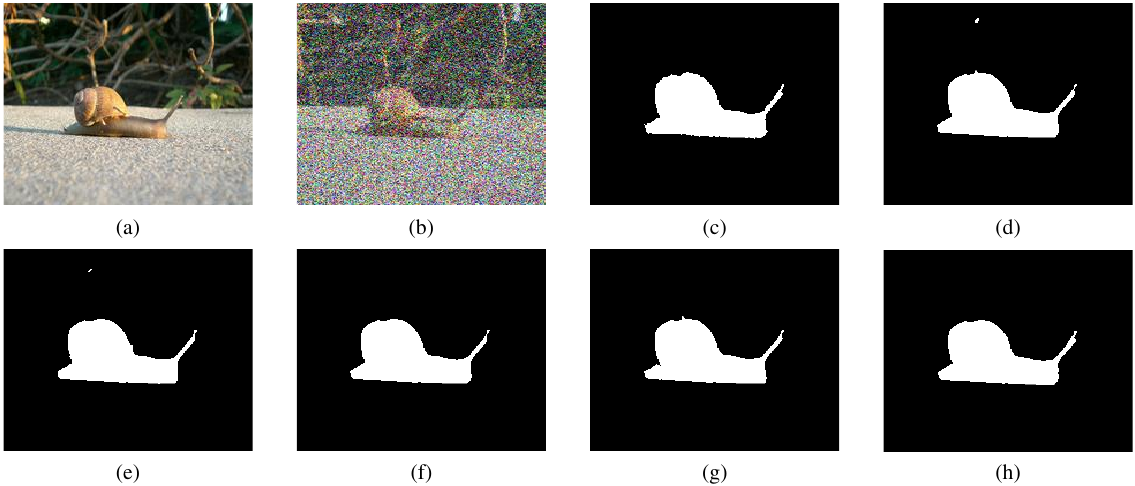}
	\caption{Results of single object of HKU-IS (Gaussian noise $\sigma = 0.7$). (a) Original image; (b) Noisy image; (c) Ground Truth; (d) UNet; (e) UNet++; (f) Swin-UNet; (g) DN; (h) Ours.}
	\label{HKUIS SD0.7}
\end{figure}

In Figure~\ref{HKUIS SD0.5 1}, it can be observed that the segmentation results of UNet and UNet++ appear to be affected by lake reflections and moderate Gaussian noise, leading to some extraneous regions. Compared to Swin-UNet and our method, DN shows slightly less accuracy in capturing fine edge details. Additionally, the footwebbing of the duck in the lower part of the image is not clearly represented in the result produced by Swin-UNet. In Figure~\ref{HKUIS SD0.5 2}, it can be observed that the results of UNet and UNet++ exhibit certain missing regions to varying degrees, likely due to noise interference, while DN shows some extraneous parts in its segmentation. Although both Swin-UNet and our method achieve reasonably complete recovery overall, our approach appears to capture the structure of the dolphin’s tail fluke and rostrum with slightly greater accuracy and fidelity.

\begin{figure}[htbp]
	\centering
	\includegraphics[width=1\linewidth]{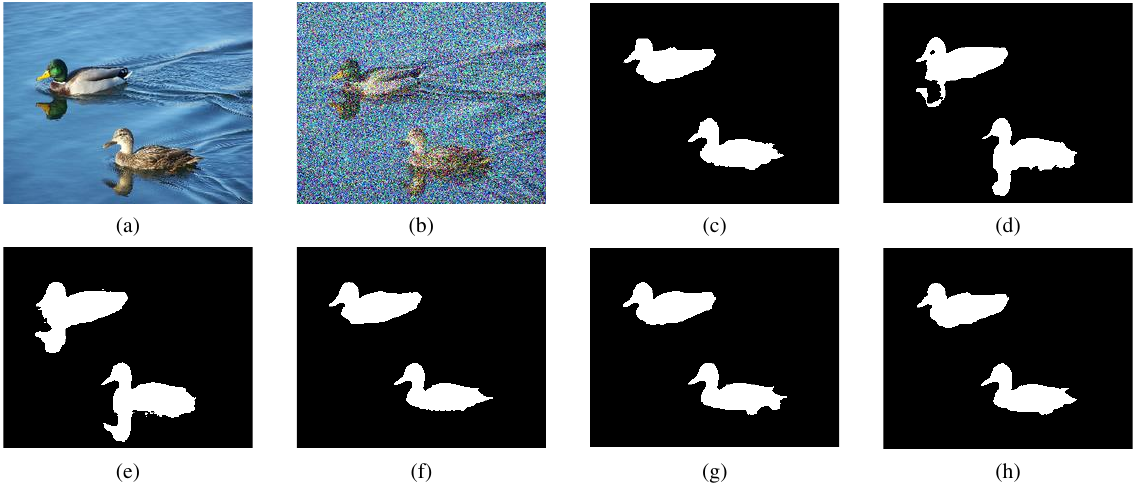}
	\caption{Results of multiple objects of HKU-IS (Gaussian noise $\sigma = 0.5$). (a) Original image; (b) Noisy image; (c) Ground Truth; (d) UNet; (e) UNet++; (f) Swin-UNet; (g) DN; (h) Ours.}
	\label{HKUIS SD0.5 1}
\end{figure}

\begin{figure}[htbp]
	\centering
	\includegraphics[width=1\linewidth]{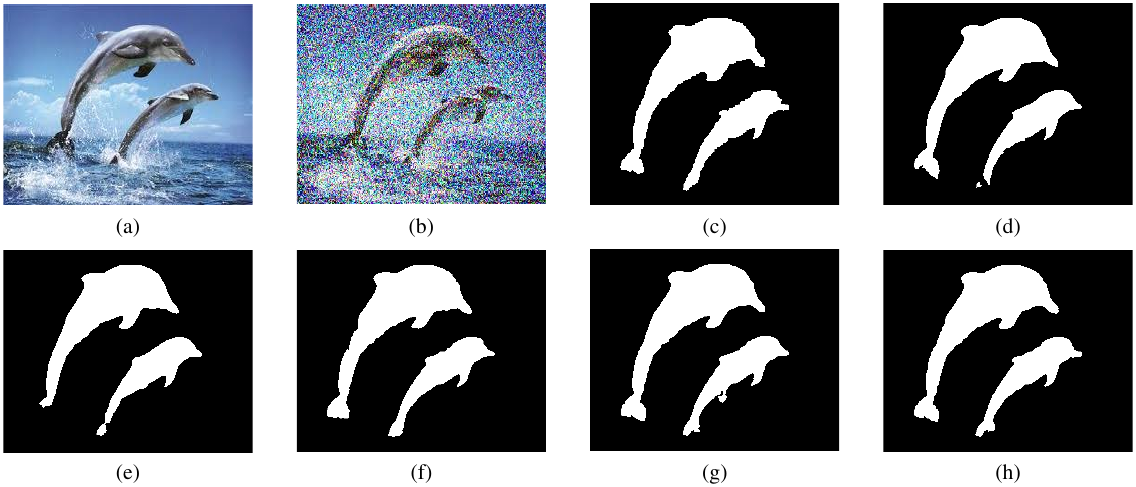}
	\caption{Results of multiple objects of HKU-IS (Gaussian noise $\sigma = 0.5$). (a) Original image; (b) Noisy image; (c) Ground Truth; (d) UNet; (e) UNet++; (f) Swin-UNet; (g) DN; (h) Ours.}
	\label{HKUIS SD0.5 2}
\end{figure}

For the results of DUT-OMRON, in Figure~\ref{DUT SD0.3}, it can be observed that the segmentation output of UNet omits the two leftmost figures and, together with UNet++, introduces some extraneous point-like artifacts. Additionally, in the case of DN, the segmentation of the rightmost figure appears to be influenced by the color similarity between the subject’s white hat and the snowy background, resulting in a partially missing hat region. Meanwhile, both Swin-UNet and our method demonstrate a degree of robustness to such interference, particularly in preserving finer details, although there remains room for further improvement. In Figure~\ref{DUT SD0.7}, it can be observed that the segmentation results of UNet and UNet++ exhibit extraneous artifacts around the head regions of the figures and completely fail to capture the outstretched hand of the figure on the right hanging outside the railing. A similar issue is present in the result of DN, which only recovers a small portion of this hand region. In comparison, both Swin-UNet and the proposed method achieve more complete and accurate segmentation overall. However, minor imperfections remain: Swin-UNet shows slight inaccuracies in some hand details, while our method exhibits subtle artifacts around the head region of the left figure.

\begin{figure}[htbp]
	\centering
	\includegraphics[width=1\linewidth]{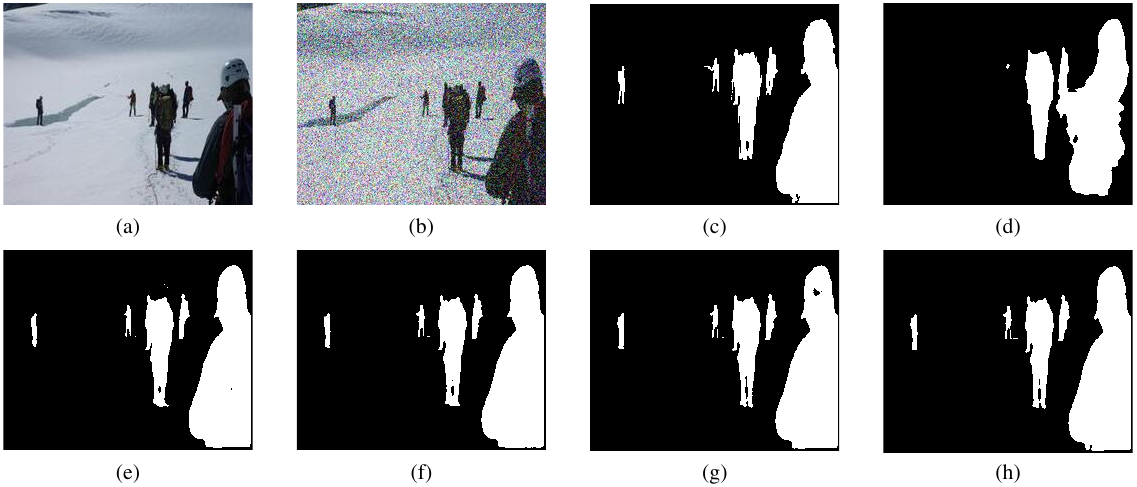}
	\caption{Results of multiple objects of DUT-OMRON (Gaussian noise $\sigma = 0.3$). (a) Original image; (b) Noisy image; (c) Ground Truth; (d) UNet; (e) UNet++; (f) Swin-UNet; (g) DN; (h) Ours.}
	\label{DUT SD0.3}
\end{figure}

\begin{figure}[htbp]
	\centering
	\includegraphics[width=1\linewidth]{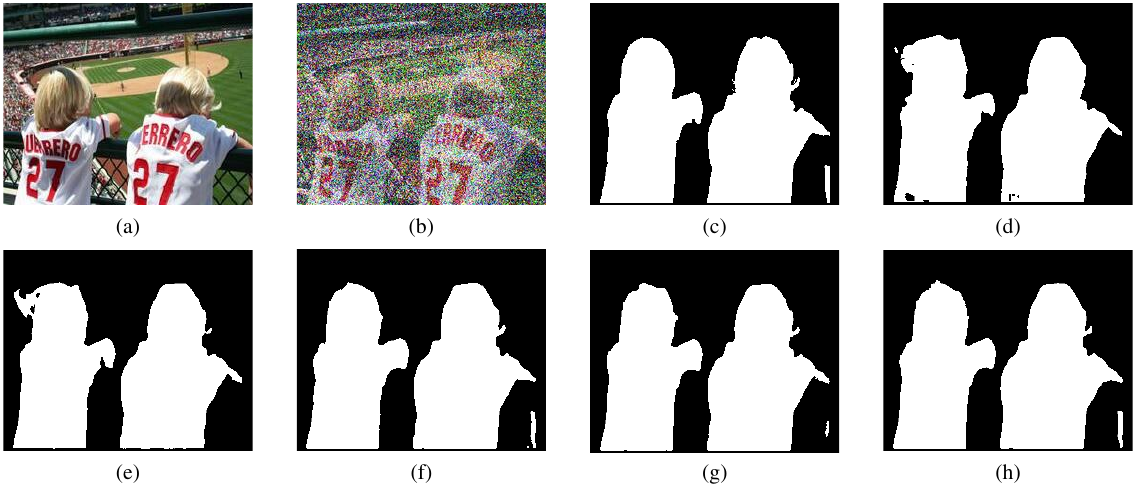}
	\caption{Results of multiple objects of DUT-OMRON (Gaussian noise $\sigma = 0.7$). (a) Original image; (b) Noisy image; (c) Ground Truth; (d) UNet; (e) UNet++; (f) Swin-UNet; (g) DN; (h) Ours.}
	\label{DUT SD0.7}
\end{figure}

\subsection{Quantitative results}
In Table~\ref{index}, we report the average quantitative results of various methods across three datasets under Gaussian noise corruption with noise intensity $\sigma=0.5$. The results suggest that the proposed method performs competitively in most scenarios, achieving comparatively higher Dice and Jaccard scores, along with a lower HD95, which may indicate improved region consistency and boundary stability under noisy conditions. Although Swin-UNet achieves competitive Dice and Jaccard scores on the DUT-OMRON dataset, its performance in boundary regions remains slightly less accurate compared to our method. Swin-UNet and DN also demonstrate solid performance in certain settings, though the proposed method appears to offer modest advantages in retaining fine structural details when subjected to noise. Together with observations from prior visual results, the proposed approach shows promising generalization ability and noise tolerance compared to several current models, suggesting its potential applicability in semantic segmentation tasks where noise interference is a concern. Its consistent performance in both regional overlap and boundary quality under impaired image conditions may support its use in real-world environments where image quality can vary.

\begin{table}
	\caption{Comparison of Dice, Jaccard and HD95 of ours with UNet, UNet++, Swin-UNet, and DN on ECSSD, HKU-IS and DUT-OMRON with Gaussian noise $\sigma=0.5$. $\uparrow$ means the higher the better; $\downarrow$ means the lower the better.}
	\label{index}
	\centering
	\resizebox{0.98\textwidth}{!}{
		\begin{tabular}{l*{10}{c}}
			\toprule
			&\multicolumn{3}{c}{ECSSD}  &
			\multicolumn{3}{c}{HKU-IS}&
			\multicolumn{3}{c}{DUT-OMRON}\\
			\cmidrule(r){2-4}  \cmidrule(r){5-7} \cmidrule(r){8-10} &Dice $\uparrow$&Jaccard  $\uparrow$&HD95 $\downarrow$&Dice $\uparrow$&Jaccard $\uparrow$ &HD95 $\downarrow$ &Dice $\uparrow$&Jaccard $\uparrow$&HD95 $\downarrow$\\
			\midrule
			\multicolumn{1}{c}{UNet}& \multicolumn{1}{c}{0.873$\pm$0.004}& \multicolumn{1}{c}{0.775$\pm$0.001}&
			\multicolumn{1}{c}{0.897$\pm$0.005}&
			\multicolumn{1}{c}{0.881$\pm$0.001}&
			\multicolumn{1}{c}{0.788$\pm$0.001}&
			\multicolumn{1}{c}{1.602$\pm$0.012}&
			\multicolumn{1}{c}{0.868$\pm$0.003}&
			\multicolumn{1}{c}{0.768$\pm$0.001}&
			\multicolumn{1}{c}{5.431$\pm$0.020}
			\\ 
			\multicolumn{1}{c}{UNet++}& \multicolumn{1}{c}{0.880$\pm$0.003}& \multicolumn{1}{c}{0.787$\pm$0.002}&
			\multicolumn{1}{c}{0.788$\pm$0.002}&
			\multicolumn{1}{c}{0.890$\pm$0.001}&
			\multicolumn{1}{c}{0.802$\pm$0.002}&
			\multicolumn{1}{c}{1.474$\pm$0.015}&
			\multicolumn{1}{c}{0.870$\pm$0.001}&
			\multicolumn{1}{c}{0.771$\pm$0.002}&
			\multicolumn{1}{c}{4.782$\pm$0.150}
			\\ 	
			\multicolumn{1}{c}{Swin-UNet}& \multicolumn{1}{c}{0.910$\pm$0.002}& \multicolumn{1}{c}{0.835$\pm$0.001}& 
			\multicolumn{1}{c}{0.498$\pm$0.002}&
			\multicolumn{1}{c}{0.906$\pm$0.001}&
			\multicolumn{1}{c}{0.829$\pm$0.003}&
			\multicolumn{1}{c}{1.001$\pm$0.007}&
			\multicolumn{1}{c}{\textbf{0.905}$\pm$0.001}&
			\multicolumn{1}{c}{\textbf{0.827}$\pm$0.004}&
			\multicolumn{1}{c}{2.451$\pm$0.020}
			\\ 
			\multicolumn{1}{c}{DN}& \multicolumn{1}{c}{0.896$\pm$0.002}& \multicolumn{1}{c}{0.812$\pm$0.002}&
			\multicolumn{1}{c}{0.563$\pm$0.001}&
			\multicolumn{1}{c}{0.897$\pm$0.002}&
			\multicolumn{1}{c}{0.813$\pm$0.002}&
			\multicolumn{1}{c}{1.287$\pm$0.010}&
			\multicolumn{1}{c}{0.885$\pm$0.001}&
			\multicolumn{1}{c}{0.794$\pm$0.002}&
			\multicolumn{1}{c}{3.239$\pm$0.030}
			\\
			\multicolumn{1}{c}{Ours}& \multicolumn{1}{c}{\textbf{0.919}$\pm$0.003}& \multicolumn{1}{c}{\textbf{0.851}$\pm$0.004}&
			\multicolumn{1}{c}{\textbf{0.432}$\pm$0.001}&
			\multicolumn{1}{c}{\textbf{0.910}$\pm$0.003}&
			\multicolumn{1}{c}{\textbf{0.835}$\pm$0.002}&
			\multicolumn{1}{c}{\textbf{0.989}$\pm$0.004}&
			\multicolumn{1}{c}{0.902$\pm$0.001}&
			\multicolumn{1}{c}{0.821$\pm$0.002}&
			\multicolumn{1}{c}{\textbf{2.364}$\pm$0.010}
			\\ 
			\bottomrule
	\end{tabular}}
\end{table}

As shown in Table~\ref{time}, the proposed method achieves reasonably competitive efficiency despite its increased architectural complexity. Our approach incorporates an $F$ module and a $T$ module, and further involves solving high-dimensional linear systems via the BiCGSTAB method—operations that typically introduce notable computational overhead. Nevertheless, the observed running time remains within a comparable order of magnitude to several current models. Specifically, our method requires 23.46 seconds per epoch, which is higher than UNet (5.55s) and UNet++ (7.36s), yet it remains relatively close to transformer-based Swin-UNet (10.87s) and DN (11.05s). These results suggest that the introduced modules do not excessively impair computational efficiency, while contributing to improved model capability. Thus, the computational cost of our approach appears acceptable, offering a favorable trade-off between model expressiveness and inference time. It performs noticeably better than UNet and UNet++, is marginally better than DN, and remains competitive with Swin-UNet.

\begin{table}
	\caption{Comparison of running time of UNet, UNet++, Swin-UNet, DN, and ours on ECSSD with Gaussian noise $\sigma = 0.5$ (s: Represents seconds per epoch).}
	\label{time}
	\centering
	\begin{tabular}{ccccc}
		\toprule
		\multicolumn{5}{c}{Running time per epoch on ECSSD} \\
		\midrule
		UNet&UNet++&Swin-UNet&DN&Ours \\
		5.54s&7.36s&10.87s&11.05s&23.46s\\
		\bottomrule
	\end{tabular}
\end{table}

\section{Conclusions}\label{conclusions}
Numerical evaluations under different levels of Gaussian noise indicate that our method may achieve modest improvements in edge sharpness, detail retention, and overall segmentation accuracy compared to contemporary CNN-based approaches, while performing comparably to transform-based techniques. It is worth noting that our model appears to maintain encouraging robustness and adaptability, even under high noise conditions. When compared with conventional neural network designs, the robust VM\_TUNet still maintains relatively moderate parameter size and computational requirements, which we hope may support its use in limited resource environments and for images with broken, blurry, or slender boundaries. Relative to widely used neural network models such as UNet, UNet++, and Swin-UNet, our approach seeks to offer somewhat better interpretability grounded in mathematical and physical insights. In comparison with hybrid frameworks like DN, the adoption of higher-order modeling may assist in preserving boundary shapes more faithfully. Future work might explore extending this model to more challenging tasks, such as instance segmentation and 3D medical image processing, while also striving to develop a more systematic theoretical understanding of its broader properties. We hope these efforts can eventually provide more practical segmentation solutions and theoretical perspectives for real-world industrial and medical applications.

\newpage
\bibliographystyle{abbrv}
\bibliography{references} 

\end{document}